\documentclass[11pt]{article}

\usepackage{setspace}

\usepackage[round,comma]{natbib}
\usepackage{graphicx} % more modern
\usepackage{algorithm}
\usepackage{algorithmic}

\usepackage{hyperref}

\usepackage[colorinlistoftodos, textwidth=40mm, shadow]{todonotes}

\definecolor{PalePurp}{rgb}{0.66,0.57,0.66}

\usepackage{amsthm}
\usepackage{amsmath}
\usepackage{amssymb}
\usepackage{bm}
\usepackage{eufrak}

\newtheorem{theorem}{Theorem}
\newtheorem{assumption}{Assumption}
\newcommand{\OO}{\mathcal{O}}
\newcommand{\EE}[1]{\mathbb{E}\left[#1\right]}

\newcommand{\EEc}[2]{\mathbb{E}\left[\left.#1\right|#2\right]}
\newcommand{\PPc}[2]{\mathbb{P}\left[\left.#1\right|#2\right]}

\newcommand{\real}{\mathbb{R}}
\newcommand{\II}{\mathbb{I}}
\newcommand{\Ind}[1]{\mathbb{I}_{\{#1\}}}
\newcommand{\ev}[1]{\left\{#1\right\}}

\newcommand{\X}{\mathcal{X}}
\renewcommand{\S}{\mathcal{S}}

\newcommand{\A}{\mathcal{A}}

\newcommand{\bW}{W}
\newcommand{\bu}{\mathbf{u}}
\newcommand{\bv}{\mathbf{v}}
\newcommand{\hbv}{\widehat{\mathbf{v}}}
\newcommand{\hbP}{\widehat{\mathbf{P}}}
\newcommand{\hbr}{\widehat{\mathbf{r}}}

\newcommand{\ra}{\rightarrow}

\newtheorem{lemma}{Lemma}

\newcommand{\bpi}{\bm{\pi}}
\newcommand{\bphi}{\bm{\phi}}
\newcommand{\bmu}{\bm{\mu}}
\newcommand{\hbmu}{\widehat{\bm{\mu}}}

\newcommand{\btau}{\bm{\tau}}
\newcommand{\bt}{\mathbf{t}}
\newcommand{\bw}{w}

\newcommand{\bttheta}{\widetilde{\bm{\theta}}}

\newcommand{\btlambda}{\widetilde{\bm{\lambda}}}

\newcommand{\bTheta}{\bm{\Theta}}
\newcommand{\bLambda}{\bm{\Lambda}}

\renewcommand{\P}{{\mathcal{P}}}
\newcommand{\R}{{\mathcal{R}}}
\newcommand{\br}{\mathbf{r}}

\newcommand{\ba}{\mathbf{a}}
\newcommand{\bs}{\mathbf{s}}
\newcommand{\bd}{\mathbf{d}}
\newcommand{\bc}{\mathbf{c}}

\newcommand{\bp}{\mathbf{p}}

\newcommand{\regret}{\mathfrak{R}}
\newcommand{\expregret}{\widehat{\regret}}

\newcommand{\argmax}{\mathop{\rm arg\, max}}

\newcommand{\bM}{\mathbf{M}}
\newcommand{\bN}{\mathbf{N}}
\newcommand{\bS}{\mathbf{S}}

\newcommand{\abs}[1]{\left\vert#1\right\vert}

\newcommand{\one}[1]{{\mathbb I}_{\{#1\}}}           % Characteristic function

\title{Online learning in MDPs with side information}
\author{Yasin Abbasi-Yadkori and Gergely Neu}

\begin{document} 

\maketitle

\begin{abstract} 
We study online learning of finite Markov decision process (MDP) problems when a side information vector is available. The problem is motivated by applications such as clinical trials, recommendation systems, etc. Such applications have an episodic structure, where each episode corresponds to a patient/customer. Our objective is to compete with the optimal dynamic policy that can take side information into account.

We propose a computationally efficient algorithm and show that its regret is at most $O(\sqrt{T})$, where $T$ is the number of rounds. To best of our knowledge, this is the first regret bound for this setting.
\end{abstract}

%%%%%%%%%%%%%%%%%%%%%%%%%%%%%%%%%%%%%%%%%%%%%%%%%%%%%%%%%%%
\section{Introduction}

% Another example is clinical trials. Here, the state space is the set of all possible responses from the patient given any sequence of treatments. The actions are the available treatments at each stage. The model is determined by the patient, and thus, changes at each trial. As we don't observe the full model at the end, the results of Section~\ref{sec:episodic-problems} are applicable. 

% MDPs with arbitrary models and reward functions

We study online learning of finite Markov decision process (MDP) problems when a side information vector is available. The problem is motivated by applications such as clinical trials~\citep{Lavori-Dawson-2000, Murphy-vanderLaan-Robins-2001}, recommendation systems~\citep{Li-Chu-Langford-Schapire-2010}, etc.

For example, consider a multi-stage treatment strategy that specifies which treatments should be applied to a patient, given his responses to the past treatments. Each patient is specified by the outcome of several tests that are performed on the patient before the treatment begins. We collect these test results in a side information vector. A simple universal strategy uses the same policy to make recommendations for all patients, independent of the side information vector. Ideally, we would like to have treatment strategies that are adapted to each patient's characteristics. 

The problem can be modeled as a MDP problem with an infinite state space. The state variable contains the side information and the patient responses up to the current stage in the treatment. Although there are regret bounds for  MDP problems with infinite state spaces~\citep{Abbasi-Yadkori-Szepesvari-2011,Abbasi-Yadkori-2012, Ortner-Ryabko-2012}, the proposed algorithms can be computationally expensive.

Alternatively, we can model the problem as a MDP problem with changing rewards and transition probability kernels. There is however no computationally efficient algorithm with a performance guarantee for this setting.

In this paper, we model such decision problems with Markov decision processes where the transition and reward functions are allowed to depend on the side information given for each new problem instance. Using our previous example, the side information corresponds to the results of the tests preceding the treatment, actions correspond to different treatment options, and the states  are given by the outcome of the applied treatments. Every new patient  corresponds to a new episode in the decision problem, where the transitions and rewards characterizing the treatment procedure are influenced by the history of the patient in question. In what follows, we precisely formulate the outlined decision problem and provide a principled way of utilizing side information to maximize rewards.

\section{Background}
To set up our goals, we need to fix some notations. Let $\|v \|$ denote the $\ell^2$ norm of vector $v$. A finite episodic Markov decision process (MDP) is characterized by its finite state space $\S$, finite action space $\A$, transition function $P$ and reward function $r$.  
An episodic MDP also has a few special states, the starting state and some terminal states: Each episode starts from the designated starting state and ends when it reaches a terminal state.
When, in addition, the state space has a layered structure with respect to the transitions, we get the so-called {\em loop-free} variant of episodic MDPs.
The layered structure of the state space means that $\S=\cup_{l=0}^L\S_l$, where $\S_l$ is called the $l$th layer of the 
state space, $\S_l\cap\S_k = \emptyset$ for all $l\neq k$, and the agent can only move between consecutive layers. That is, for any $s\in\S_l$ and $a\in\A$, $P(s'|s,a)=0$ if $s'\not\in\S_{l+1}, l=0,\ldots,L-1$.
In particular, each episode starts at layer $0$,  from state $s_0$. In every state $s_l\in\S_l$, the learner chooses an action $a_l\in\A$, earns some reward $r(s_l,a_l)$, and is eventually transferred to state $s_{l+1}\sim P(\cdot|s,a)$. The episode ends when the learner reaches any state $s_L$ belonging to the last layer $\S_L$. 
This assumption is equivalent to assuming that each trajectory consists of exactly $L$ transitions.\footnote{Note that all loop-free state spaces can be transformed to one that satisfies our assumptions with no significant increase in the size of the problem. A simple transformation algorithm is given in Appendix A~of \citet{gyorgy07sp}.} This framework is a natural fit for episodic problems where time is part of the state variable. Figure~\ref{fig:ssp} shows an example of a loop-free episodic MDP.
For any state $s\in\S$ we will use $l_s$ to denote the index of the layer $s$ belongs to, that is, $l_s=l$ if $s\in\S_l$. 

\begin{figure*}[t]
\centering
\includegraphics[width=\textwidth]{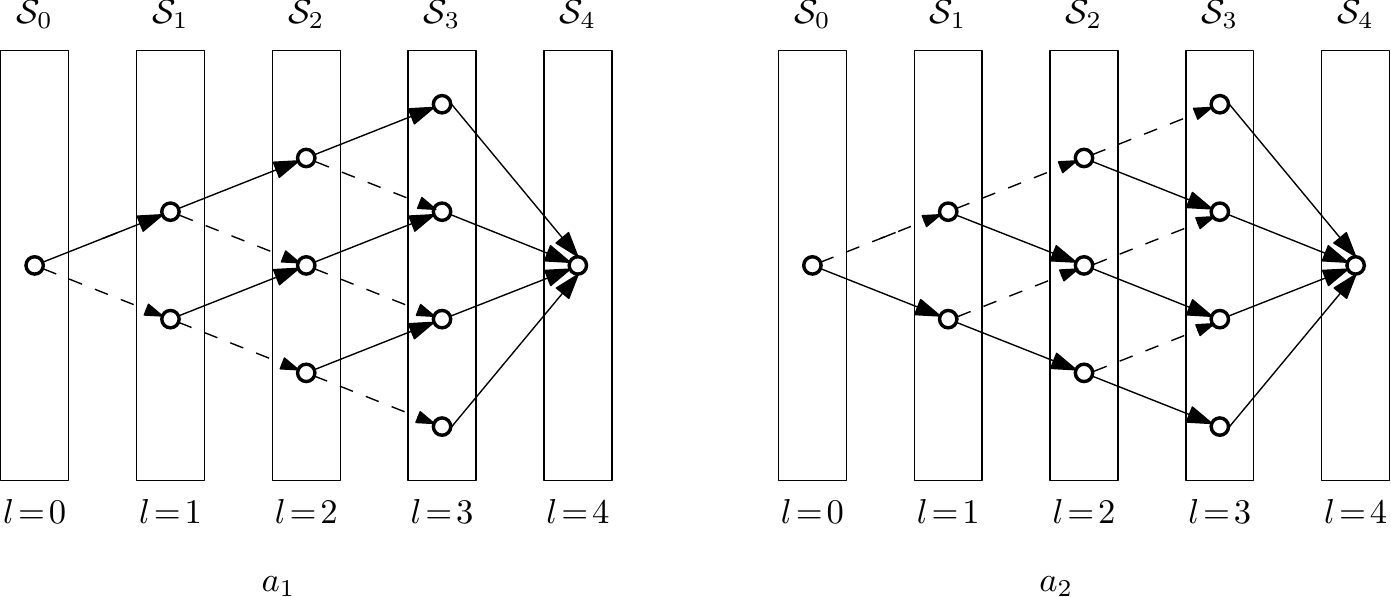}
\caption{An example of a loop-free episodic Markov decision process when two actions $a_1$ (``up'') and $a_2$ (``down'') are available in all states. Nonzero transition probabilities under each action are indicated with arrows between circles representing states. In the case of two successor states, the successor states with the intended direction with larger probabilities are connected with solid arrows, dashed arrows indicate  less probable transitions.}
\label{fig:ssp}
\end{figure*}

A deterministic policy $\pi$ (or, in short: a policy) is a mapping $\pi:\S\ra\A$. We say that a policy $\pi$ is followed in an episodic MDP problem if the action in state $s\in\S$ is set to be $\pi(s)$, independently of previous states and actions. The set of all deterministic policies will be denoted by $\Pi$. A random path $\bu = \left(\bs_0,\ba_0, \dots,\bs_{L-1},\ba_{L-1}, \bs_L\right)$ is said to be generated by policy $\pi$ under the transition model $P$ if the initial state is $\bs_0$ and $\bs_{l+1}\in\S_{l+1}$ is drawn from $P(\cdot|\bs_l,\pi(\bs_l))$ for all $l=0,1,\dots,L-1$. We denote this relation by $\bu\sim(\pi,P)$. 
Define the \emph{value} of a policy $\pi$, given a fixed reward function $r$ and a transition model $P$ as
 \[
\bW(r,\pi,P) = \EEc{\sum_{l=0}^{L-1} r(\bs_l,\pi(\bs_l))}{\bu\sim(\pi,P)},
\]
that is, the expected sum of rewards gained when following $\pi$ in the MDP defined by $r$ and $P$.

\section{The learning problem}

We consider episodic loop-free environments where the transitions and rewards are influenced by some vector $x\in D \subset\real^d$ of side information. In particular, the probability of a transition to state $s'$ given that action $a$ was chosen in state $s$ is given by the generalized linear model
\[
P_x(s'|s,a) = \sigma(\varphi(x)^\top \theta_*(s',s,a)),
\]
where $\sigma:\real\ra [0,1]$ is a \textit{link function} (such as sigmoid function), $\varphi:\real^d\ra\real^n$ is a feature mapping, and $\theta_*(s',s,a)\in \Theta \subset\real^n$ is some unknown parameter vector for each individual $(s',s,a)\in\S\times\S\times\A$. Furthermore, the rewards are parametrized as
\[
\EE{r_x(s,a)} = \sigma(\psi(x)^\top \lambda_*(s,a)),
\]
where $\psi:\real^d\ra\real^m$ is another feature mapping and $\lambda_*(s,a) \in \Lambda \subset\real^m$.

 In every episode $t=1,2,\dots,T$ of our learning problem, we are given a side information vector $x_t$, which gives rise to the reward function $r_{x_t}$ and transition functions $P_{x_t}$.
A reasonable goal in this setting is to accumulate nearly as much reward as the best dynamic policy that can take side information into account. Defining such a dynamic policy as a mapping $\phi:\real^d\ra\Pi$, we denote the best achievable performance by
\begin{equation}\label{eq:bestperf}
V_T^* = \max_{\phi:\real^d\ra \Pi} \sum_{t=1}^T W(r_{x_t},\phi(x_t),P_{x_t}).
\end{equation}
The expected value of the learner's policy $\bpi_t$ in episode $t$ will be denoted by $\bv_t = W(r_{x_t},\bpi_t,P_{x_t})$. 
We are interested in online algorithms that have no information about the parameter vectors $\theta_*$ and $\lambda_*$ at the beginning of the learning process, but minimize the following notion of \emph{regret}:
\[
\expregret_T = V_T^* - 
\sum_{t=1}^T \bv_t.
\]

%%%%%%%%%%%%%%%%%%%%%%%%%%%%%%%%%%%%%%%%%%%%%%%%%%%%%%%%%%%
\section{Algorithm}

\begin{algorithm}
\textbf{Input:} State space $\S$, action space $\A$, confidence parameter $0 < \delta <1$.

\textbf{Initialization:} 
\\
\textbf{For each episode $t=1,2,\dots,T$:}
\begin{enumerate}
\item Observe side information $x_t$.
\item Construct confidence sets according to Equations~\eqref{eq:confset_r} and \eqref{eq:confset_P} 
\item Compute $\bpi_t$, $\hbP_t$ and $\hbr_t$ according to Equation~\eqref{eq:optimistic}.  %as $\left(\bpi_t, \tbP_t\right) =\argmax_{\pi\in\Pi,\bP\in\confset_{\bi(t)}}\left\{\bW(R_{t-1} + \bY_{\bi(t)}, \pi, \bP)\right\}$.
\item Traverse path $\bu_t\sim(\bpi_t,P_{x_t})$.
\item Receive rewards $\sum_{l=0}^{L-1} r_{x_t}\bigl(\bs_l^{(t)},\ba_l^{(t)}\bigr)$.
\end{enumerate}
\caption{Algorithm for online learning in episodic MDPs with side information.}
\label{tran:alg:fullinfo_fpl}
\end{algorithm}

Our algorithm combines ideas from the UCRL2 algorithm of \citet{jaksch10ucrl} and the results of \citet{filippi10genlin}. The algorithm that we propose is based on the \textit{Optimism in the Face of Uncertainty} (OFU) principle. First proposed by \citet{Lai-Robbins-1985}, OFU is a general principle that can be employed to design efficient algorithms in many stochastic online learning problems~\citep{Auer-Cesa-Bianchi-Fischer-02,Auer-2002,Dani-Hayes-Kakade-2008,Abbasi-Yadkori-Szepesvari-2011}. The basic idea is to maintain a confidence set for the unknown parameter vector and then in every round choose an estimate from the confidence set together with a policy so that the predicted expected reward is maximized, i.e., the estimate-policy pair is chosen optimistically. 

%Since we do not know the transition function $P$ and the reward function $r$ a priori, we have to rely on some estimates of them. 
%The idea is to 
To implement the OFU principle, we construct confidence sets $\bTheta_t$ and $\bLambda_t$ that contain the true models $\theta_*$ and $\lambda_*$ with probability at least $1-\delta$ each. The confidence parameter $\delta\in(0,1)$ is specified by the user.
Our parametric estimates take the form
\[
P^\theta_x(s'|s,a) = \sigma(\varphi(x)^\top \theta(s',s,a))
\]
and
\[
r^\lambda_x(s,a) = \sigma(\psi(x)^\top \lambda(s,a))
\]
for each $x\in\real^d$ and $(s',s,a)\in\S\times\S\times\A$. Using these notations, the confidence sets $\bTheta$ and $\bLambda$ translate to  confidence sets for the transition and reward functions as
\[
\P_x(\bTheta) = \ev{P_x^\theta: \theta\in\bTheta}\quad\mbox{and}\quad\R_x(\bLambda) = \ev{r_x^\lambda: \lambda\in\bLambda}.
\]
Using these confidence sets, we select our model and policy simultaneously as
\begin{equation}\label{eq:optimistic}
\left(\bpi_t,\hbP_t,\hbr_t\right) = \argmax_{\stackrel{\bpi_t\in\Pi}{P_{x_t}^\theta\in\P_{x_t}(\bTheta),r_{x_t}^\lambda\in\R_{x_t}(\bLambda)}}
W\left(r^\lambda_{x_t},\bpi_t,P^\theta_{x_t}\right).
\end{equation}
The above optimization task can be efficiently performed by the extended dynamic programming algorithm presented in Section~\ref{sec:ext_DP} (see also \citealp{neu12ssp-trans}). 

We employ techniques similar to \citet{filippi10genlin} to construct our confidence sets. Let 
\[
\bN_{t,s,a,s'} = I + \sum_{u=1}^{t-1} \one{\bs_{l_s}^{(u)}=s, \ba_{l_s}^{(u)}=a, \bs_{l_s+1}^{(u)}=s'} \varphi(x_u) \varphi(x_u)^\top
\]
and 
\[
\bM_{t,s,a} = I+ \sum_{u=1}^{t-1} \one{\bs_{l_s}^{(u)}=s, \ba_{l_s}^{(u)}=a} \psi(x_u) \psi(x_u)^\top \;. 
\]
Let $\bS_t(s,a)$ be the set of time steps up to time $t$ that $(s,a)$ is observed. At time $t$, we solve the equations
\begin{align}
\label{eq:lambda}
\sum_{u\in \bS_t(s,a)} (r_{x_u}(s,a) - \sigma(\psi(x_u)^\top \lambda(s,a) ) ) &= 0 \,, \\
\label{eq:theta}
\sum_{u\in \bS_t(s,a)} (\one{\bs_{l_{s'}}^{(u)} = s' } - \sigma(\varphi(x_u)^\top \theta(s,a,s') ) ) &= 0 \,,
\end{align}
to obtain $\btlambda_t(s,a)$ and $\bttheta_t(s,a,s')$. Let $\rho_t$ be an increasing function (to be specified later). Then, the confidence interval corresponding with $r_{x_t}(s,a)$ at time $t$ is $[\br^-_t(s,a), \br^+_t(s,a)]$, where
\[
\br^-_t(s,a) = \sigma(\psi(x_t)^\top \btlambda_t(s,a)) - \rho_t \| \psi(x_t) \|_{\bM_{t,s,a}^{-1}} \,,
\]
and
\[
\br^+_t(s,a) =  \sigma(\psi(x_t)^\top \btlambda_t(s,a)) + \rho_t \| \psi(x_t) \|_{\bM_{t,s,a}^{-1}}\;.
\]
Similarly, the confidence interval corresponding with $P_{x_t}(s'|s,a)$ at time $t$ is given by $[\bp^-_t(s'|s,a),\bp^+_t(s'|s,a)]$, where
\[
\bp^-_t(s'|s,a) = \sigma(\varphi(x_t)^\top \bttheta_t(s',s,a)) - \rho_t \| \varphi(x_t) \|_{\bN_{t,s,a,s'}^{-1}} \,,
\]
and
\[
\bp^+_t(s'|s,a) =  \sigma(\varphi(x_t)^\top \bttheta_t(s',s,a)) + \rho_t \| \varphi(x_t) \|_{\bN_{t,s,a,s'}^{-1}}\;.
\]
Summarizing, our confidence sets for the reward and transition functions are respectively defined as
\begin{equation}\label{eq:confset_r}
\R_{x_t}(\bLambda_t) = \ev{r:r(s,a)\in
[\br^-_t(s,a),\br^+_t(s,a)]}
\end{equation}
and
\begin{equation}\label{eq:confset_P}
\P_{x_t}(\bTheta_t) = \ev{P:P(s'|s,a)\in
[\bp^-_t(s'|s,a),\bp^+_t(s'|s,a)]}.
\end{equation}

%%%%%%%%%%%%%%%%%%%%%%%%%%%%%%%%%%%%%%%%%%%%%%%%%%%%%%%%%%%
\section{Analysis}

First, we make a number of assumptions. 
\begin{assumption}
\label{ass:Lipschitz}
Function $\sigma:\real\ra\real$ is continuously differentiable, Lipschitz with constant $k_{\sigma}$. Further, we have that $c_{r} = \inf_{x\in D, \theta\in \Theta} \dot{\sigma} (\varphi(x)^\top \theta) > 0$, and $c_{P} = \inf_{x\in D, \lambda\in \Lambda} \dot{\sigma} (\psi(x)^\top \lambda) > 0$, where $\dot\sigma$ denotes the derivative of $\sigma$. 
\end{assumption}
\begin{assumption}
\label{ass:visitor-boundedness}
There exists $L>0$ such that for all $x\in D$, $\|x \| \le L$.
\end{assumption}
\begin{assumption}
\label{ass:reward-bounded}
Function $\sigma:\real\ra\real$ is bounded in $[0,1]$. %Further, $\EE{r_t - \sigma(\psi(x_t)^\top \lambda) | F_{t-1} } = 0$. \todoY{fix this}
\end{assumption}
The main result of this section is the following theorem.
\begin{theorem}\label{thm:main}
Let Assumptions~\ref{ass:Lipschitz},\ref{ass:visitor-boundedness},\ref{ass:reward-bounded} hold. Then, with probability at least $1-\delta$, for any sequence of side information vectors,\footnote{We use $\widetilde{\OO}$ to hide logarithmic factors in the big-O notation.}
\[
\expregret_T = \widetilde{\OO}\left(L|\S|^2|\A|(n+m)\sqrt{T}\right).
\]
%\begin{align*}
%\expregret_T &\le (L+1) |\S| \sqrt{2\, T\, \log\frac{L}{\delta}} \\ 
%&+ 2 \beta_{T,1}(\delta) L |\S|^2 |\A| \sqrt{ T n \log \left( 1+ \frac{T L^2}{d} \right) }\\ 
%&+ 2 \beta_{T,2}(\delta) |\S| |\A| \sqrt{ T m \log\left( 1 + \frac{T L^2}{d} \right) } \;.
%\end{align*}
\end{theorem}

We will need a number of lemmas to prove the theorem. 
\begin{lemma}[\citet{filippi10genlin}, Proposition 1]
\label{lem:conf-intervals}
Take any $\delta,t$ such that $0<\delta<\min\left(1,\frac{m}{e}, \frac{n}{e} \right)$ and $1+\max(m,n,2) \le t \le T$. Let $\kappa = \sqrt{3 + 2 \log (1+2L^2)}$. Let
\begin{align*}
\beta_{r,t}(\delta) &= \frac{2 k_{\sigma} \kappa }{c_{r}} \sqrt{2 n \log (t) \log(n/\delta)}\,,\\
\beta_{P,t}(\delta) &= \frac{2 k_{\sigma} \kappa }{c_{P}} \sqrt{2 m \log (t) \log(m/\delta)}\; .
\end{align*}
Let $\btlambda_t(s,a)$ and $\bttheta_t(s,a,s')$ be the solutions of \eqref{eq:lambda} and \eqref{eq:theta}, respectively. Then, for any $x\in\real^d$, any $z=(s,a)$, and any $s'\in \S_{l_s+1}$, with probability at least $1-\delta$, it holds that
\begin{align*}
\abs{ \sigma\left(\varphi(x)^\top \theta_*(z,s')\right) - \sigma\left(\varphi(x)^\top \bttheta_t(z,s') \right) } \le\beta_{r,t}(\delta) \| \varphi(x) \|_{\bN_{t,z,s'}^{-1}}\;.
\end{align*}
Also, with probability at least $1-\delta$,
\begin{align*}
\abs{ \sigma\left(\psi(x)^\top \lambda_*(z)\right) - \sigma\left(\psi(x)^\top \btlambda_t(z) \right) } \le\beta_{P,t}(\delta) \| \psi(x) \|_{\bM_{t,z}^{-1}}\;.
\end{align*}
\end{lemma}
\begin{lemma}[\citet{Abbasi-Yadkori-2012}, Lemma E.3]
\label{lem:dir-var}
Let $\{w_1,\dots, w_t\}$ be a sequence in $\real^k$. Define $W_s = I + \sum_{\tau=1}^{s-1} w_\tau w_\tau^\top$. If $\|w_\tau \|\le L$ for all $\tau$, then
\[
\sum_{s=1}^t \min\left(1,\|w_s \|_{W_s^{-1}}^2\right) \le 2 k \log \left( 1 + \frac{t L^2}{k} \right) \; .
\] 
\end{lemma}

Let $\phi^*$ be the dynamic policy achieving the maximum in \eqref{eq:bestperf}. Also, let $\hbv_t = W\left(\hbr_t,\bpi_t,\hbP_t\right)$ be the estimate of the learner's value $\bv_t$. 
Since we select our model and policy optimistically, we have $\hbv_t \ge W(r_{x_t},\bphi^*_{x_t},P_{x_t})$. 
It follows that the regret can be bounded as 
\begin{align}
\notag
\expregret_T =& V_T^* - \sum_{t=1}^T \bv_t
\\
\notag
=& V_T^* - \sum_{t=1}^T \hbv_t + \sum_{t=1}^T \left(\hbv_t - \bv_t\right)
\\
\label{eq:regret-decomp}
\le& \sum_{t=1}^T \left(\hbv_t - \bv_t\right).
\end{align}
To treat this term, we use some results by \citet{neu12ssp-trans}. Consider $\bmu_t(s) = \PPc{\bs_{l_s} = s}{\bu\sim(\bpi_t,P_{x_t})}$, that is, the probability that a trajectory generated by $\bpi_t$ and $P_{x_t}$ includes $s$. Note that given a layer $\X_l$, the restriction $\bmu_{t,l}:\X_l\ra[0,1]$ is a distribution. Define an estimate of $\bmu_t$ as $\hbmu_t(s) = \PPc{\bs_l = s}{\bu\sim(\bpi_t,\hbP_t)}$. 
%Note that this estimate can be efficiently computed using the recursion
%\begin{equation*}
%\hbmu_t(s_{l+1})= \sum_{s_{l},a_{l}} P_{x_t}^{\bhtheta_t}(s_{l+1}|s_{l},a_{l})\bpi_t(a_{l}|s_{l}) \hbmu_t(s_{l}),
%\label{eq:recurse}
%\end{equation*}
%for $l=0,1,2,\dots,L-1$, with $\hbmu_t(x_0)=1$.
First, we repeat Lemma~4 of \citet{neu12ssp-trans}.
\begin{lemma}\label{lem:p_to_mu}
Assume that there exists some function $\bd_t: \S\times\A\ra \mathbb{R}^+$ such that $\left\|\hbP_t(\cdot|s,a) - P_{x_t}(\cdot|s,a)\right\|_1 \le \bd_t(s,a)$ 
holds for all $(s,a)\in\S\times\A$.
Then
\[
\sum_{s_l\in\S_l}|\hbmu_t(s_l) - \bmu_t(s_l)| \le \sum_{k=0}^{l-1} \sum_{s_k\in\S_k} \bmu_t(s_k)\,  \bd_t\left(s_k, \bpi_t(s_k)\right)
\]
for all $l=1,2,\dots,L-1$.
\end{lemma}

Using this result, we can prove the following statement. 
\begin{lemma}\label{lem:trans_to_v}
Assume that there exist some functions $\bd_t: \S\times\A \ra \mathbb{R}^+$ and $\bc_t: \S\times\A \ra \mathbb{R}^+$ such that $\left\|\hbP_t(\cdot|s,a) - P_{x_t}(\cdot|s,a)\right\|_1 \le \bd_t(s,a)$ 
and $\left|\hbr_t(s,a) - r_{x_t}(s,a)\right| \le \bc_t(s,a)$ 
hold for all $(s,a)\in\S\times\A$ with probability at least $1-\delta$ each. Then with probability at least $1-4\delta$,
\[
\begin{split}
\sum_{t=1}^T (\hbv_t - \bv_t) \le& \sum_{t=1}^T \sum_{l=0}^{L-1} L\bd_t\left(\bs_l^{(t)}, \ba_l^{(t)}, x_t\right) +  \sum_{t=1}^T \sum_{l=0}^{L-1} \bc_t\left(\bs_l^{(t)}, \ba_l^{(t)}, x_t\right) + (L+1) |\S| \sqrt{2\, T\, \log\frac{L}{\delta}}.
\end{split}
\]
\end{lemma}
\begin{proof}
Fix an arbitrary $t:1\le t\le T$. We have $\hbv_t = \sum_{l=0}^{L-1} \sum_{s\in\S_l}\hbmu_t(s) \hbr_t(s,\bpi_t(s))$
and $\bv_t = \sum_{l=0}^{L-1} \sum_{s\in\S_l}\bmu_t(s) r_{x_t}(s,\bpi_t(s))$, 
thus
\[
\begin{split}
\hbv_t - \bv_t =& \sum_{l=0}^{L-1} \sum_{s\in\S_l} \left(\hbmu_t(s) - \bmu_t(s)\right) \hbr_t(s,\bpi_t(s)) 
+\sum_{l=0}^{L-1} \sum_{s\in\S_l} \bmu_t(s) \left(\hbr_t(s,\bpi_t(s)) - r_{x_t}(s,\bpi_t(s))\right)
\\
\le& \sum_{l=0}^{L-1} \sum_{s\in\S_l} \left|\hbmu_t(s) - \bmu_t(s)\right|
+\sum_{l=0}^{L-1} \sum_{s\in\S_l} \bmu_t(s) \left(\hbr_t(s,\bpi_t(s)) - r_{x_t}(s,\bpi_t(s))\right)
\end{split}
\]
Using our upper bound on $\left\|\hbP_t(\cdot|s,a) - P(\cdot|s,a)\right\|_1$ for all $(s,a) \in \S\times\A$ along with Lemma~\ref{lem:p_to_mu}, we get
\begin{equation}\label{eq:mubound}
\begin{split}
\sum_{s\in\S_l} \left|\hbmu_t(s) - \bmu_t(s)\right| &\le \sum_{k=0}^{l-1} \sum_{s_k\in\S_k} \bmu_t(s_k)\,  \bd_t\left(s_k, \bpi_t\left(s_k\right)\right) \\
&=  \sum_{k=0}^{l-1} \bd_t\left(\bs_k^{(t)}, \ba_k^{(t)}\right) + \sum_{k=0}^{l-1} \sum_{s_k\in\S_k}\left(\bmu_t(s_k) - \Ind{\bs_{k}^{(t)} = s_k}\right) \bd_t\left(s_k,\bpi_t\left(s_k\right)\right)
\end{split}
\end{equation}
with probability at least $1-\delta$.
Similarly, using our upper bound on $\left|\hbr_t(s,a) - r_{x_t}(s,a)\right|$ for all $(s,a) \in \S\times\A$, we get
\begin{equation}\label{eq:rbound}
\begin{split}
\sum_{l=0}^{L-1} \sum_{s\in\S_l} \bmu_t(s) \left(\hbr_t(s,\bpi_t(s)) - r_{x_t}(s,\bpi_t(s))\right) &\le  \sum_{l=0}^{L-1} \bc_t\left(\bs_l^{(t)}, \ba_l^{(t)}\right) \\ 
&\qquad+ \sum_{l=0}^{L-1} \sum_{s_l\in\S_l}\left(\bmu_t(s_l) - \Ind{\bs_{l}^{(t)} = s_l}\right) \bc_t\left(s_l,\bpi_t\left(s_l\right)\right)
\end{split}
\end{equation}
with probability at least $1-\delta$.
For the second term on the right hand side of \eqref{eq:mubound}, notice that
$
\left(\bmu_t(s_k) - \II_{\left\{\bs_{k}^{(t)} = s_k\right\}}\right) 
$
form a martingale difference sequence with respect to $\{\bu_t\}_{t=1}^T$ and thus by the Hoeffding--Azuma inequality and $\bd_t\le 2$ almost surely, we have
\[
\sum_{t=1}^T\left(\bmu_t(s_k) - \II_{\left\{\bs_{k}^{(t)} = s_k\right\}}\right) \bd_t\left(s_k,\bpi_t\left(s_k\right)\right) \le \sqrt{2\, T\, \log\frac{ L}{\delta}}
\]
with probability at least $1-\delta/L$.
The union bound implies that we have, with probability at least $1-2\delta$ simultaneously for all $l=1,\ldots,L$, 
\begin{equation}
\begin{split}
\sum_{t=1}^T \sum_{s\in\S_l} \left|\hbmu_t(s) - \bmu_t(s)\right| &\le \sum_{t=1}^T \sum_{k=0}^{l-1} \bd_t\left(\bs_k^{(t)}, \ba_k^{(t)}\right)  + \sum_{k=0}^{l-1}|\S_k| \sqrt{2\, T\, \log\frac{L}{\delta}}  \\
&\le \sum_{t=1}^T \sum_{k=0}^{L-1} \bd_t\left(\bs_k^{(t)}, \ba_k^{(t)}\right) + |\S| \sqrt{2\, T\, \log\frac{L}{\delta}}.
\label{eq:muboundSum}
\end{split}
\end{equation}
By a similar argument, 
\begin{equation}
\begin{split}
\sum_{t=1}^T \sum_{l=0}^{L-1} \sum_{s\in\S_l} \bmu_t(s) \left(\hbr_t(s,\bpi_t(s)) - r_{x_t}(s,\bpi_t(s))\right) \le \sum_{t=1}^T \sum_{l=0}^{L-1} \bc_t\left(\bs_l^{(t)}, \ba_l^{(t)}\right)  + |\S| \sqrt{2\, T\, \log\frac{L}{\delta}}  \\
\label{eq:rboundSum}
\end{split}
\end{equation}
also holds with probability at least $1-\delta$. We obtain the statement of the lemma by using the union bound.
\end{proof}

Now we are ready to prove our main result.

\begin{proof}[Proof of Theorem~\ref{thm:main}]
Fix some $l$. Let $D_{l,T} = \sum_{t=1}^T\bd_t\left(\bs_l^{(t)}, \ba_l^{(t)}\right)$  
and $C_{l,T} = \sum_{t=1}^T\bc_t\left(\bs_l^{(t)}, \ba_l^{(t)}\right)$. Fix $z=(s,a)$. Let $\btau(t)$ be the number of time steps that we have observed $z$ up to time $t$. Let
\[
\bd_t(s,a) = \beta_{r,\btau(t)}(\delta) \sum_{s'}  \| \varphi(x) \|_{\bN_{\tau(t),s,a,s'}^{-1}}.
\]
By Lemma~\ref{lem:conf-intervals}, $\bd_t(s,a)$ is an upper bound on the error of our transition estimates, thus satisfying the condition of Lemmas~\ref{lem:p_to_mu} and \ref{lem:trans_to_v}.

Let $\bt(\tau)$ be the timestep that we observe $(s,a)$ for the $\tau$th time. Notice that $\btau(\bt(\tau)) = \tau$. Let $\epsilon_\tau^2 =  \min\left(  \| \varphi(x_{\bt(\tau)}) \|_{\bN_{\tau,z,s'}^{-1}}^2,1 \right)$. As $\bd_t \le 2$, we can write
\begin{align}
\notag
D_{l,T} &= 2 \sum_{z\in (\S_l, \A)} \one{\bs_l^{(t)}=s, \ba_l^{(t)}=a} \min\{\bd_t \left(z\right) ,1 \} \\
\notag
&= 2\sum_{z\in (\S_l, \A)} \sum_{\tau=1}^{\btau(T)} \min\{\bd_{t(\tau)} \left(z\right) ,1 \} \\
\notag
&= \sum_{z\in (\S_l, \A)} \sum_{\tau=1}^{\btau(T)} \beta_{r,\tau}(\delta) \sum_{s'} \epsilon_\tau \\
\notag
&\le \beta_{r,T}(\delta) \sum_{z\in (\S_l, \A)} \sum_{s'} \sum_{\tau=1}^{\btau(T)}  \epsilon_\tau \\
\notag
&\le \beta_{r,T}(\delta) \sum_{z\in (\S_l, \A)} \sum_{s'} \sqrt{ \btau(T) \sum_{\tau=1}^{\btau(T)}  \epsilon_\tau^2 } \\
\label{eq:D}
&\le 2 \beta_{r,T}(\delta) |\S_l||\S_{l+1}| |\A| \sqrt{ T n \log \left( 1+ \frac{T L^2}{n} \right) } \;,
\end{align}
where the last inequality follows from Lemma~\ref{lem:dir-var}. Similarly, we can prove that
\begin{equation}
\label{eq:C}
C_{l,T} \le 2 \beta_{P,T}(\delta) |\S_l| |\A| \sqrt{ T m \log\left( 1 + \frac{T L^2}{m} \right) } \;.
\end{equation}
Summing up these bounds for all $l=0,1,\dots,L-1$, using Inequality~\eqref{eq:regret-decomp} and Lemma~\ref{lem:trans_to_v} gives the upper bound on the regret as
\begin{align*}
\expregret_T &\le (L+1) |\S| \sqrt{2\, T\, \log\frac{L}{\delta}} + 2 \beta_{P,T}(\delta) L |\S|^2 |\A| \sqrt{ T n \log \left( 1+ \frac{T L^2}{d} \right) } \\
&\qquad+ 2 \beta_{R,T}(\delta) |\S| |\A| \sqrt{ T m \log\left( 1 + \frac{T L^2}{d} \right) } \;.
\end{align*}
\end{proof}

\section{Extended dynamic programming}\label{sec:ext_DP}

The extended dynamic programming algorithm is given by Algorithm~\ref{alg:ext_DP}.
\begin{algorithm*}
\textbf{Input:} confidence sets of the form \eqref{eq:confset_r} and \eqref{eq:confset_P} .

\textbf{Initialization:} Set $\bw(s_L) = 0$.

\textbf{For $l=L-1,L-2,\dots,0$}
\begin{enumerate}
\item Let $k=|\S_{l+1}|$ and $\left(s^*_1,s^*_2,\dots,s^*_k\right)$ be a sorting of the states in $\X_{l+1}$ such that $\bw(s_1^*) \ge \bw(s_2^*) \ge \dots \ge \bw(s_k^*)$.
\item \textbf{For all $(s,a)\in \S_l\times\A$}
\begin{enumerate}
\item $r^*(s,a) = \min\left\{r^+(s,a),\, 1\right\}$.
\item $\Delta(s,a) = \sum_{i=2}^k p^-(s^*_i|s,a)$.
\item $P^*(s_1^*|s,a) = \min\left\{p^+(s'|s,a)-\Delta(s,a),\, 1\right\}$.
\item $P^*(s_i^*|s,a) = \sigma(\varphi(x_t)^\top \bttheta_t(s_i^*,s,a))$ for all $i=2,3,\dots,k$.
\item Set $j=k$.
\item \textbf{While $\sum_{i} P^*(s^*_i|s,a)>1$ do}
\begin{enumerate}
\item Set $P^*(s^*_j|s,a) = \max\left\{p^-(s_j^*|s,a), 1- \sum_{i\neq j} P^*(s^*_i|s,a)\right\}$
\item Set $j=j-1$.
\end{enumerate}
\end{enumerate}
\item \textbf{For all $s\in\S_l$}
\begin{enumerate}
\item Let $\bw(s) = \max_a \left\{r^*(s,a) + \sum_{s'} P^*(s'|s,a) \bw(s')\right\}$.
\item Let $\pi^*(s) = \argmax_a \left\{r^*(s,a) + \sum_{s'} P^*(s'|s,a) \bw(s')\right\}$.
\end{enumerate}
\end{enumerate}

\textbf{Return:} optimistic transition function $P^*$, optimistic reward function $r^*$, optimistic policy $\pi^*$.
\caption{Extended dynamic programming for finding an optimistic policy and transition model for a given confidence set of transition functions and given rewards.}
\label{alg:ext_DP}
\end{algorithm*}
The next lemma, which can be obtained by a straightforward modification of the proof of Theorem~7 of \citet{jaksch10ucrl}, shows that Algorithm~\ref{alg:ext_DP} efficiently solves the desired minimization problem.

\bigskip

\begin{lemma}
\label{lem:ext_DP}
Algorithm~\ref{alg:ext_DP} solves the maximization problem \eqref{eq:optimistic} for the confidence sets $\P_{x}(\bTheta)$ and $\R_{x}(\bLambda)$. Let 
$C=\sum_{l=0}^{L-1} |\S_l||\S_{l+1}|$ denote the maximum number of possible transitions in the given model. The time and space complexity of Algorithm~\ref{alg:ext_DP} is the number of possible non-zero elements of $P$ allowed by the given structure, and so it is $\OO(C|\A|)$, which, in turn, is $\OO(|\A| |\S|^2)$.
\end{lemma}

\section{Conclusions}

In this paper, we introduced a model for online learning in episodic MDPs where the transition and reward functions can depend on some side information provided to the learner. We proposed and analyzed a novel algorithm for minimizing regret in this setting and have shown that the regret of this algorithm is $\widetilde{\OO}(L|\S|^2|\A|(n+m)\sqrt{T})$. While we are not aware of any theoretical results for this precise setting, it is beneficial to compare our results to previously known guarantees for other settings. 

First, the UCRL2 algorithm of \citet{jaksch10ucrl} enjoys a regret bound of $\widetilde{\OO}(L|\S|\sqrt{|\A|T})$ for $L$-step episodic problems with fixed transition and reward functions. This setting can be regarded as a special case of ours when $m=n=1$ and constant (or i.i.d.) side information, thus our bounds for this case are worse by a multiplicative factor of $\OO(|\S|\sqrt{|\A|})$. \footnote{This difference stems from the fact that we have to directly bound the error of $\left|\hbP_t(s'|s,a)-P_{x_t}(s'|s,a)\right|$ instead of the norm $\left\|\hbP_t(\cdot|s,a)-P_{x_t}(\cdot|s,a)\right\|_1$. While such a bound is readily available for the single-parameter linear setting, it is highly non-trivial whether a similar result is provable for generalized linear models.} However, our algorithm achieves low regret against a much richer class of policies, so our guarantees are far superior when side information has a large impact on the decision process.

There is also a large literature on temporal-difference methods that are model-free and learn a value function. Asymptotic behavior of temporal-difference methods \citep{Sutton-1998} in large state and action spaces is studied both in on-policy \citep{Tsitsiklis-Van-Roy-1997} and off-policy \citep{Sutton-Szepesvari-Maei-2009, Sutton-Maei-Precup-Bhatnagar-Silver-Szepesvari-Wiewiora-2009, Maei-Szepesvari-Bhatnagar-Precup-Silver-Sutton-2009} settings. All these results concern the policy estimation problem, i.e., estimating the value of a fixed policy. The available results for the control problem, i.e., estimating the value of the optimal policy, are more limited \citep{Maei-Szepesvari-Bhatnagar-Sutton-2010} and prove only convergence to local optimum of some objective function. It is not clear if and under what conditions these TD control methods converge to the optimal policy.

\citet{yu09ArbitraryRewardsTransitions, yu09Modulated} consider the problem of online learning in MDPs where the reward and transition functions may change arbitrarily after each time step. Their setting can be seen as a significantly more difficult version of ours, when the side information $x_t$ is only  revealed \emph{after} the learner selects its policy $\bpi_t$. One cannot expect to be able to compete with the set of dynamic policies using side information, so they consider regret minimization against the pool of stationary state-feedback policies. Still, the algorithms proposed in these papers fail to achieve sublinear regret. The mere existence of consistent learning algorithms for this problem is a very interesting open problem. 

An interesting direction of future work is to consider learning in unichain Markov decision processes where a new side information vector is provided after each transition made in the MDP. The main challenge in this setting is that long-term planning in such a quickly changing environment is very difficult without making strong assumptions on the generation of the sequence of side information vectors. Learning in the situation when the sequence $(x_t)$ is generated by an oblivious adversary is not much simpler than in the setting of \citet{yu09ArbitraryRewardsTransitions, yu09Modulated}: seeing one step into the future does not help much when having to plan multiple steps ahead in a Markovian environment.

We expect that a non-trivial combination of the ideas presented in the current paper with principles of online prediction of arbitrary sequences can help constructing algorithms that achieve consistency in the above settings.

\bibliography{bib}

\begin{thebibliography}{22}
\providecommand{\natexlab}[1]{#1}
\providecommand{\url}[1]{\texttt{#1}}
\expandafter\ifx\csname urlstyle\endcsname\relax
  \providecommand{\doi}[1]{doi: #1}\else
  \providecommand{\doi}{doi: \begingroup \urlstyle{rm}\Url}\fi

\bibitem[Abbasi-Yadkori(2012)]{Abbasi-Yadkori-2012}
Y.~Abbasi-Yadkori.
\newblock \emph{Online Learning for Linearly Parametrized Control Problems}.
\newblock PhD thesis, University of Alberta, 2012.

\bibitem[Abbasi-Yadkori and
  Szepesv\'{a}ri(2011)]{Abbasi-Yadkori-Szepesvari-2011}
Y.~Abbasi-Yadkori and Cs. Szepesv\'{a}ri.
\newblock Regret bounds for the adaptive control of linear quadratic systems.
\newblock In \emph{COLT}, 2011.

\bibitem[Auer(2002)]{Auer-2002}
Peter Auer.
\newblock Using confidence bounds for exploitation-exploration trade-offs.
\newblock \emph{Journal of Machine Learning Research}, 3:\penalty0 397--422,
  2002.

\bibitem[Auer et~al.(2002)Auer, Cesa-Bianchi, and
  Fischer]{Auer-Cesa-Bianchi-Fischer-02}
Peter Auer, Nicol\`o Cesa-Bianchi, and Paul Fischer.
\newblock Finite time analysis of the multiarmed bandit problem.
\newblock \emph{Machine Learning}, 47\penalty0 (2-3):\penalty0 235--256, 2002.

\bibitem[Dani et~al.(2008)Dani, Hayes, and Kakade]{Dani-Hayes-Kakade-2008}
Varsha Dani, Thomas~P. Hayes, and Sham~M. Kakade.
\newblock Stochastic linear optimization under bandit feedback.
\newblock In \emph{Conference on Learning Theory}, pages 355--366, 2008.

\bibitem[Filippi et~al.(2010)Filippi, Capp{\'e}, Garivier, and
  Szepesv{\'a}ri]{filippi10genlin}
Sarah Filippi, Olivier Capp{\'e}, Aur{\'e}lien Garivier, and Csaba
  Szepesv{\'a}ri.
\newblock Parametric bandits: The generalized linear case.
\newblock In \emph{NIPS}, pages 586--594, 2010.

\bibitem[Gy\"{o}rgy et~al.(2007)Gy\"{o}rgy, Linder, Lugosi, and
  Ottucs\'{a}k]{gyorgy07sp}
Andr\'{a}s Gy\"{o}rgy, Tam\'{a}s Linder, G\'{a}bor Lugosi, and
  {\text{Gy}}\"{o}rgy Ottucs\'{a}k.
\newblock The on-line shortest path problem under partial monitoring.
\newblock \emph{Journal of Machine Learning Research}, 8:\penalty0 2369--2403,
  2007.
\newblock ISSN 1532-4435.

\bibitem[Jaksch et~al.(2010)Jaksch, Ortner, and Auer]{jaksch10ucrl}
Thomas Jaksch, Ronald Ortner, and Peter Auer.
\newblock Near-optimal regret bounds for reinforcement learning.
\newblock \emph{J. Mach. Learn. Res.}, 99:\penalty0 1563--1600, August 2010.
\newblock ISSN 1532-4435.
\newblock URL \url{http://portal.acm.org/citation.cfm?id=1859890.1859902}.

\bibitem[Lai and Robbins(1985)]{Lai-Robbins-1985}
Tze~Leung Lai and Herbert Robbins.
\newblock Asymptotically efficient adaptive allocation rules.
\newblock \emph{Advances in Applied Mathematics}, 6:\penalty0 4--22, 1985.

\bibitem[Lavori and Dawson(2000)]{Lavori-Dawson-2000}
P.~W. Lavori and R.~Dawson.
\newblock A design for testing clinical strategies: biased individually
  tailored within-subject randomization.
\newblock \emph{Journal of the Royal Statistical Society A}, 163:\penalty0
  29--38, 2000.

\bibitem[Li et~al.(2010)Li, Chu, Langford, and
  Schapire]{Li-Chu-Langford-Schapire-2010}
Lihong Li, Wei Chu, John Langford, and Robert~E. Schapire.
\newblock A contextual-bandit approach to personalized news article
  recommendation.
\newblock In \emph{WWW}, 2010.

\bibitem[Maei et~al.(2009)Maei, Szepesv\'{a}ri, Bhatnagar, Precup, Silver, and
  Sutton]{Maei-Szepesvari-Bhatnagar-Precup-Silver-Sutton-2009}
H.~R. Maei, Cs. Szepesv\'{a}ri, S.~Bhatnagar, D.~Precup, D.~Silver, and R.~S.
  Sutton.
\newblock Convergent temporal-difference learning with arbitrary smooth
  function approximation.
\newblock In \emph{Advances in Neural Information Processing Systems}, 2009.

\bibitem[Maei et~al.(2010)Maei, Szepesv\'{a}ri, Bhatnagar, and
  Sutton]{Maei-Szepesvari-Bhatnagar-Sutton-2010}
H.~R. Maei, Cs. Szepesv\'{a}ri, S.~Bhatnagar, and R.~S. Sutton.
\newblock Toward off-policy learning control with function approximation.
\newblock In \emph{Proceedings of the 27th International Conference on Machine
  Learning}, 2010.

\bibitem[Murphy et~al.(2001)Murphy, van~der Laan, and
  Robins]{Murphy-vanderLaan-Robins-2001}
S.~A. Murphy, M.~J. van~der Laan, and J.~M. Robins.
\newblock Marginal mean models for dynamic regimes.
\newblock \emph{Journal of American Statistical Association}, 96:\penalty0
  1410--1423, 2001.

\bibitem[Neu et~al.(2012)Neu, Gy\"orgy, and Szepesv\'ari]{neu12ssp-trans}
Gergely Neu, Andr\'as Gy\"orgy, and {{Cs}}aba Szepesv\'ari.
\newblock The adversarial stochastic shortest path problem with unknown
  transition probabilities.
\newblock In \emph{Proceedings of the Fifteenth International Conference on
  Artificial Intelligence and Statistics}, volume~22 of \emph{JMLR Workshop and
  Conference Proceedings}, pages 805--813, La Palma, Canary Islands, April
  21-23 2012.

\bibitem[Ortner and Ryabko(2012)]{Ortner-Ryabko-2012}
R.~Ortner and D.~Ryabko.
\newblock Online regret bounds for undiscounted continuous reinforcement
  learning.
\newblock In \emph{NIPS}, 2012.

\bibitem[Sutton et~al.(2009{\natexlab{a}})Sutton, Maei, Precup, Bhatnagar,
  Silver, Szepesv\'{a}ri, and
  Wiewiora]{Sutton-Maei-Precup-Bhatnagar-Silver-Szepesvari-Wiewiora-2009}
R.~S. Sutton, H.~R. Maei, D.~Precup, S.~Bhatnagar, D.~Silver, Cs.
  Szepesv\'{a}ri, and E.~Wiewiora.
\newblock Fast gradient-descent methods for temporal-difference learning with
  linear function approximation.
\newblock In \emph{Proceedings of the 26th International Conference on Machine
  Learning}, 2009{\natexlab{a}}.

\bibitem[Sutton et~al.(2009{\natexlab{b}})Sutton, Szepesv\'{a}ri, and
  Maei]{Sutton-Szepesvari-Maei-2009}
R.~S. Sutton, Cs. Szepesv\'{a}ri, and H.~R. Maei.
\newblock A convergent {O}(n) algorithm for off-policy temporal-difference
  learning with linear function approximation.
\newblock In \emph{Advances in Neural Information Processing Systems},
  2009{\natexlab{b}}.

\bibitem[Sutton(1988)]{Sutton-1998}
Richard~S. Sutton.
\newblock Learning to predict by the methods of temporal differences.
\newblock \emph{Machine Learning}, 3:\penalty0 9--44, 1988.

\bibitem[Tsitsiklis and Van~Roy(1997)]{Tsitsiklis-Van-Roy-1997}
John~N. Tsitsiklis and Benjamin Van~Roy.
\newblock An analysis of temporal-difference learning with function
  approximation.
\newblock \emph{IEEE TRANSACTIONS ON AUTOMATIC CONTROL}, 42\penalty0
  (5):\penalty0 674--690, 1997.

\bibitem[Yu and Mannor(2009{\natexlab{a}})]{yu09ArbitraryRewardsTransitions}
Jia~Yuan Yu and Shie Mannor.
\newblock Online learning in {M}arkov decision processes with arbitrarily
  changing rewards and transitions.
\newblock In \emph{GameNets'09: Proceedings of the First ICST international
  conference on Game Theory for Networks}, pages 314--322, Piscataway, NJ, USA,
  2009{\natexlab{a}}. IEEE Press.
\newblock ISBN 978-1-4244-4176-1.

\bibitem[Yu and Mannor(2009{\natexlab{b}})]{yu09Modulated}
Jia~Yuan Yu and Shie Mannor.
\newblock Arbitrarily modulated {M}arkov decision processes.
\newblock In \emph{Joint 48th IEEE Conference on Decision and Control and 28th
  Chinese Control Conference}, pages 2946--2953. IEEE Press,
  2009{\natexlab{b}}.

\end{thebibliography}

\end{document}